\documentclass{article}
\usepackage{times}
\usepackage{hyperref}
\usepackage{graphicx}
\usepackage{amsmath,amssymb,amsbsy,amsfonts,amsthm}
\makeatletter
\renewcommand*{\eqref}[1]{%
  \hyperref[{#1}]{\textup{\tagform@{\ref*{#1}}}}%
}
\makeatother
\usepackage[utf8]{inputenc}
\usepackage{biblatex}
\AtEveryBibitem{
	\clearlist{language}
	\clearfield{doi}
	\clearfield{issn}
	\clearfield{isbn}
	\clearfield{urlyear}
	\ifentrytype{misc}{}{\clearfield{url}}
}
\usepackage{tikz}
\usetikzlibrary{calc,arrows,arrows.meta}
\usepackage{array}
\usepackage[text={17.5cm,25cm},centering]{geometry}
\newcommand{\titleS}[2]{\title{{#2}}}
\newtheorem{definition}{Definition}

\newtheorem{proposition}{Proposition}
\newcommand{\eqnrel}[1]{\hspace{.2cm}{#1}\hspace{.2cm}}
\newcommand{\tuple}[1]{\langle{#1}\rangle}
\DeclareMathOperator{\rank}{rank}

\titleS{arXiv:1905.01289 and NeurIPS 2019 workshop on Graph Representation Learning}{Convolution, attention and structure embedding}
\author{Jean-Marc Andreoli\footnote{\tt jean-marc.andreoli@naverlabs.com}\\NAVER LABS Europe, Grenoble, France\\
{\small\tt http://www.europe.naverlabs.com}}
\date{April 2019, last modified March 2020}
\begin{document}
\maketitle
\begin{abstract}
Deep neural networks are composed of layers of parametrised linear operations intertwined with non linear activations. In basic models, such as the multi-layer perceptron, a linear layer operates on a simple input vector embedding of the instance being processed, and produces an output vector embedding by straight multiplication by a matrix parameter. In more complex models, the input and output are structured and their embeddings are higher order tensors. The parameter of each linear operation must then be controlled so as not to explode with the complexity of the structures involved. This is essentially the role of convolution models, which exist in many flavours dependent on the type of structure they deal with (grids, networks, time series etc.). We present here a unified framework which aims at capturing the essence of these diverse models, allowing a systematic analysis of their properties and their mutual enrichment. We also show that attention models naturally fit in the same framework: attention is convolution in which the structure itself is adaptive, and learnt, instead of being given a priori.
\end{abstract}
\section{A generic framework for convolution on arbitrary structures}
Convolution is a powerful operator, which is widely used in deep neural networks in many different flavours:~\cite{lecun_convolutional_1998,lecun_convolutional_2010,ian_goodfellow_deep_2016,dumoulin_guide_2016,kipf_semi-supervised_2016,shi_convolutional_2015,mallat_understanding_2016}. It allows to express in a compact form operations on a structured bundle of similarly shaped data instances (embeddings of nodes in a network, of instants in a time series, of pixels in an image, etc.) taking into account some known structural dependencies between them (edges between nodes, or temporal relations between instants, or positional relations between pixels). In spite of their apparent diversity, these structures can be formalised as {\em families} of weighted graphs, where each graph in a family captures one aspect of the structure. We develop a generic model of convolution over such structures.
\subsection{Some useful properties of tensors}
\label{sec:tensors}
A tensor is characterised by its shape $S{=}\tuple{S_1\cdots S_{|S|}}$, which is a sequence of integers, its index set which is the cartesian product $\bar{S}\triangleq\prod_{i=1:|S|}\{1\cdots S_i\}$ of cardinality $|\bar{S}|{=}\prod_{i=1:|S|}S_i$, and its value which is a mapping from its index set into the set of scalars. By construction, the space of tensors of a given shape $S$ is of dimension $|\bar{S}|$. If $S$ and $T$ are shapes, we let $ST$ denote their concatenation. The following common operations on tensors are recalled here (the notation $a{:}S$ stands for ``tensor $a$ of shape $S$''):
\[
{\setlength{\extrarowheight}{2pt}
\begin{array}{|l|l|l|l|}
\hline
& \textrm{operands} & \textrm{result} & \textrm{definition} \\
\hline
\textrm{\em slicing} &  a:ST\hspace{.3cm}s\in\bar{S} & a_s:T &
(a_s)_t \triangleq a_{st}\\
\hline
\textrm{\em flattening} &
a:ST\hspace{.3cm}\omega:\bar{S}\mapsto\{1\cdots K\}\textrm{ \small bijective, hence }K=|\bar{S}|& 
a^{[\omega]}:\tuple{K}T &
a^{[\omega]}_{\tuple{k}t} \triangleq a_{(\omega^{-1}k)t}\\
\hline
\textrm{\em outer product} &  a:S\hspace{.3cm}b:T & a\otimes b:ST &
(a\otimes b)_{st} \triangleq a_sb_t\\
\hline
\end{array}}
\]
In the case of flattening, when $T$ is of length $1$ (resp. $0$), then $\boldsymbol{a}^{[\omega]}$ is a matrix (resp. a vector) and flattening is then called {\em matricisation} (resp. {\em vectorisation})~\cite{rabanser_introduction_2017}. A common choice for $\omega$ is the {\em canonical bijection}~\cite{rabanser_introduction_2017} $\omega_S$ defined for each $s{\in}\bar{S}$ by
\begin{align*}
\omega_S(s) & \eqnrel{\triangleq} 1+{\sum}_{i=1:|S|}(s_i-1){\prod}_{j=i+1:|S|}S_j
\end{align*}
In this paper, we also make use of a less common operation on tensors, called here the {\em mixed} product: if $K$ is an integer and $\boldsymbol{a},\boldsymbol{b}$ are tensors of shape $\tuple{K}S$ and $\tuple{K}T$, respectively, their mixed product denoted $\boldsymbol{a}\circ\boldsymbol{b}$ is a tensor of shape $ST$ defined by
\begin{align}
\label{eqn:mixed-product}
\boldsymbol{a}\circ\boldsymbol{b} & \eqnrel{\triangleq} {\sum}_k\boldsymbol{a}_k\otimes\boldsymbol{b}_k
\end{align}
Operator $\circ$ combines features of both the inner and outer products. It can be seen as a partially factorised form for tensors. Full tensor factorisation, according to the CanDecomp scheme~\cite{rabanser_introduction_2017}, corresponds to the case where each $\boldsymbol{a}_k$ and each $\boldsymbol{b}_k$ is itself of rank $1$, i.e. decomposed into an outer product of $|S|$ and $|T|$ {\em vectors}, respectively. The rank of the resulting factorised form is, in that case, bounded by $K$, a fact which is exploited by many low rank approximation schemes. But even in its partial form, the factorisation of Equation~\eqref{eqn:mixed-product} satisfies a rank constraint, which can also be used to model low-rank approximations of arbitrary tensors:
\[
\rank(\boldsymbol{a}\circ\boldsymbol{b})\leq{\sum}_k\rank(\boldsymbol{a}_k)\rank(\boldsymbol{b}_k)
\]
\begin{proposition}[Inversion]
\label{prop:inversion}
Let $S,T$ be arbitrary shapes and $K$ an integer. Let $\boldsymbol{a}$ be a tensor of shape $\tuple{K}S$ such that the family $(\boldsymbol{a}_k)_{k=1:K}$ be a basis of the space of tensors of shape $S$ (hence $K{=}|\bar{S}|$). Then for any tensor $\boldsymbol{\Phi}$  of shape $ST$ there exists a unique tensor $\boldsymbol{\Theta}$ of shape $\tuple{K}T$ such that $\boldsymbol{\Phi}=\boldsymbol{a}\circ\boldsymbol{\Theta}$.
\end{proposition}
\begin{proof}
Observe that $\boldsymbol{\Theta}\mapsto\boldsymbol{a}\circ\boldsymbol{\Theta}$ is a linear mapping from the space of tensors of shape $\tuple{K}T$ into the space of tensors of shape $ST$. The assumption ($\boldsymbol{a}$ is a basis) implies that it is injective, and since the two spaces have the same dimension, the mapping is an isomorphism.
\end{proof}
The expression $\sum_k\boldsymbol{a}_k{\otimes}\boldsymbol{\Theta}_k$ is formally similar to a linear combination of the basis tensors $(\boldsymbol{a}_k)_{k=1:K}$, except that the coefficients $\boldsymbol{\Theta}_k$ are tensors and scalar multiplication is replaced by tensor product. In the special case where $S$ and $T$ are both of length $1$, then $\boldsymbol{\Phi},\boldsymbol{a}$ and $\boldsymbol{\Theta}$ are all matrices, and Proposition~\ref{prop:inversion} states that, if $\boldsymbol{a}$ is invertible, any matrix $\boldsymbol{\Phi}$ can be factorised as $\boldsymbol{\Phi}{=}\boldsymbol{a}^\top\boldsymbol{\Theta}$ in a unique way. This is indeed straightforward, and $\boldsymbol{\Theta}$ has, in that case, a simple form: $\boldsymbol{\Theta}{=}\boldsymbol{a}^{-1\top}\boldsymbol{\Phi}$. The general case of arbitrary $S,T$ can be derived from that special case by first flattening the tensors involved into matrices, then using the special case to state the property on those matrices, and finally reformulating it in the original tensor space by ``un-flattening''. In other words, Proposition~\ref{prop:inversion} is nothing but an ``un-flattened'' form of matrix inversion.
\subsection{Convolutions as factorised linear transforms}
In a convolution layer, the input does not consist of a simple embedding vector, as in a standard linear layer. Instead, it is a matrix $\boldsymbol{x}$ of shape $\tuple{M,P}$, representing a bundle of $M$ entries encoded as vectors of shape $\tuple{P}$. Similarly, the output $\boldsymbol{y}$ is a matrix of shape $\tuple{N,Q}$ ($N$ entries encoded with shape $\tuple{Q}$). For example, in image convolutions, $M,P$ are the number of pixels and channels, respectively, of the input image, while $N,Q$ are those of the output image. More generally $\boldsymbol{x}$ and $\boldsymbol{y}$ could be tensors --- e.g. images are usually thought of as ternary tensors --- but a tensor can always be flattened into a matrix, or even a vector (see Section~\ref{sec:tensors}). Matricisation, rather than full vectorisation, is used here in order to keep separate the uncontrolled, structural dimensions (width and height in images, of size $M$ in input and $N$ in output) from the controlled ones (channels, of size $P$ in input and $Q$ in output). By analogy with a simple linear layer, the most general form of a convolution layer is an arbitrary linear transform, given by
\begin{align}
\label{eqn:linear}
\boldsymbol{y}_{nq} & \eqnrel{=} {\sum}_{mp}\boldsymbol{x}_{mp}\boldsymbol{\Phi}_{mnpq}
\end{align}
Tensor $\boldsymbol{\Phi}$, of shape $\tuple{M,N,P,Q}$, induces (linear) dependencies between each component of each input entry in $\boldsymbol{x}$ and each component of each output entry in $\boldsymbol{y}$. Using an arbitrary $\boldsymbol{\Phi}$ directly as parameter of the convolution is not satisfactory. First, its shape depends on the numbers $M,N$ of input and output entries: $M,N$ may vary for different instances of the data, or may be too large to be involved in the size of a parameter\footnote{The dependence on $P,Q$, on the other hand, is not problematic, since these are hyper-parameters controlled by the model (embedding sizes).}. Furthermore, in Equation~\eqref{eqn:linear}, the structural dependencies between the $M$ input and $N$ output entries are not captured.

We propose to capture this structure as a tensor $\boldsymbol{A}$ of shape $\tuple{K,M,N}$, for some integer $K$, and to constrain $\boldsymbol{\Phi}$ to be in the partially factorised, low rank form:
\begin{align}
\label{eqn:factorisation}
\boldsymbol{\Phi} & \eqnrel{=} \boldsymbol{A}\circ\boldsymbol{\Theta}
\hspace{.5cm}\left(={\sum}_k\boldsymbol{A}_k\otimes\boldsymbol{\Theta}_k\right)
\end{align}
where $\boldsymbol{\Theta}$ is a tensor of shape $\tuple{K,P,Q}$. Integer $K$ is assumed to be a hyper-parameter controlled by the model, so $\boldsymbol{\Theta}$ has a fully controlled shape and is chosen as parameter of the convolution. Tensor $\boldsymbol{A}$ on the other hand characterises the structure underlying the convolution, and can be viewed as a family $(\boldsymbol{A}_k)_{k=1:K}$ of matrices (weighted graphs between input and output entries). The variety of existing convolution mechanisms derives from various choices for $K$ and $\boldsymbol{A}$ (called resp. the {\em size} and {\em basis} of the convolution), which obey different intuitions in different domains. Examples are given below. But in general, combining Equations~\eqref{eqn:linear} and~\eqref{eqn:factorisation} together, we obtain a formula for convolution over arbitrary structures:
\begin{proposition}
A {\em convolution} of basis $\boldsymbol{A}$, a tensor of shape $\tuple{K,M,N}$, and parameter $\boldsymbol{\Theta}$, a tensor of shape $\tuple{K,P,Q}$, is a linear transform which maps a bundle of inputs $\boldsymbol{x}$ represented as a matrix of shape $\tuple{M,P}$, into a bundle of outputs $\boldsymbol{y}$ represented as a matrix of shape $\tuple{N,Q}$, according to the rule
\begin{align}
\label{eqn:convolution}
\boldsymbol{y} & \eqnrel{=} {\sum}_k\boldsymbol{A}_k^\top\boldsymbol{x}\boldsymbol{\Theta}_k
\end{align}
\end{proposition}
Note that our model of structural dependencies is flexible. If $(\boldsymbol{A}_k)_{k=1:K}$ is taken to be a basis of the whole space of matrices of shape $\tuple{M,N}$, then by Proposition~\ref{prop:inversion} any $\boldsymbol{\Phi}$ can be written as $\boldsymbol{A}\circ\boldsymbol{\Theta}$, and the resulting class of convolutions is the class of arbitrary linear transforms. But of course, this assumes $K{=}MN$, which is uncontrolled. At the other end of the spectrum, if $K{=}1$ and $\boldsymbol{A}_1$ is the identity matrix, the input entries are processed identically and fully independently, leading to a degenerate class of convolutions also known in the image domain as $1{\times}1$ convolutions. In fact, Equation~\eqref{eqn:factorisation} can be viewed as a truncated, low-rank version of the factorisation of $\boldsymbol{\Phi}$ defined by Proposition~\ref{prop:inversion} where family $(\boldsymbol{A}_k)_{k=1:K}$ is seen as a subset of a basis (of the whole space of matrices of shape $\tuple{M,N}$), of which the other members are ignored. $(\boldsymbol{A}_k)_{k=1:K}$ act as ``principal components''. 
\begin{proposition}
\label{prop:stack}
Given two convolutions of size $K',K''$, basis $\boldsymbol{A}',\boldsymbol{A}''$, parameter $\boldsymbol{\Theta}',\boldsymbol{\Theta}''$, respectively, their composition, when the dimensions match (i.e. $\tuple{N',Q'}{=}\tuple{M'',P''}$), is a convolution of size $K$, basis $\boldsymbol{A}$, parameter $\boldsymbol{\Theta}$ where
\[
K = K'K''
\hspace{1cm}
\boldsymbol{A}_{\omega(k',k'')} = \boldsymbol{A}'_{k'}\boldsymbol{A}''_{k''}
\hspace{1cm}
\boldsymbol{\Theta}_{\omega(k',k'')} = \boldsymbol{\Theta}'_{k'}\boldsymbol{\Theta}''_{k''}
\]
and $\omega$ is a bijective mapping $\{1{\cdots}K'\}{\times}\{1{\cdots}K''\}{\mapsto}\{1{\cdots}K\}$, e.g. the canonical bijection $\omega_{\tuple{K',K''}}$.
\end{proposition}
\begin{proof}
Simple application of Equation~\eqref{eqn:convolution}.
\end{proof}
\begin{figure}
\begin{center}
\begin{tikzpicture}[scale=.8]
\node[label={[label distance=-9pt,opacity=.5]270:{\small $\boldsymbol{Y}$: output}}] (bnq) at (0,0) {$bnq$};
\node[label={[label distance=-5pt]90:{$\boldsymbol{A}$: basis}}] (kmn) at (0,3) {$kmn$};
\node[label={[label distance=-7pt]240:{$\boldsymbol{X}$: input}}] (bmp) at ({-3*sqrt(3)/2},-1.5) {$bmp$};
\node[label={[label distance=-7pt]330:{$\boldsymbol{\Theta}$: parameter}}] (kpq) at ({3*sqrt(3)/2},-1.5) {$kpq$};
\node (bknp) at ($(bmp)!.5!(kmn)$) {$bknp$};
\node (mnpq) at ($(kmn)!.5!(kpq)$) {$mnpq$};
\node (bkmq) at ($(bmp)!.5!(kpq)$) {$bkmq$};
\draw [->,green] (bmp) -- (bknp); \draw [->,green] (kmn) -- (bknp); \draw [->,green] (kpq) -- (bnq); \draw [->,green] (bknp) -- (bnq);
\draw [->,blue] (kmn) -- (mnpq); \draw [->,blue] (kpq) -- (mnpq); \draw [->,blue] (bmp) -- (bnq); \draw [->,blue] (mnpq) -- (bnq);
\draw [->,red] (bmp) -- (bkmq); \draw [->,red] (kpq) -- (bkmq); \draw [->,red] (kmn) -- (bnq); \draw [->,red] (bkmq) -- (bnq);

\end{tikzpicture}
\end{center}
\caption{\label{fig:computing}A representation of three alternatives (red-green-blue, each starting from one side of the triangle) to compute a convolution $\boldsymbol{Y}$ (order-3 tensor at the centre). The vertices of the triangle are the order-3 tensors involved ($\boldsymbol{X}$: input, $\boldsymbol{\Theta}$: parameter, $\boldsymbol{A}$: basis) with their respective indices ($b$: batch index, $m/n$: input/output entry, $p/q$: input/output channel, $k$: basis index). The arrows represent sum-product operations in the so called Einstein's notation. For example, the two arrows $bmp,kpq\rightarrow bkmq$ (bottom) represent an operation yielding the order-4 tensor $R_{bkmq}=\sum_{p}\boldsymbol{X}_{bmp}\boldsymbol{\Theta}_{kpq}$.}
\end{figure}
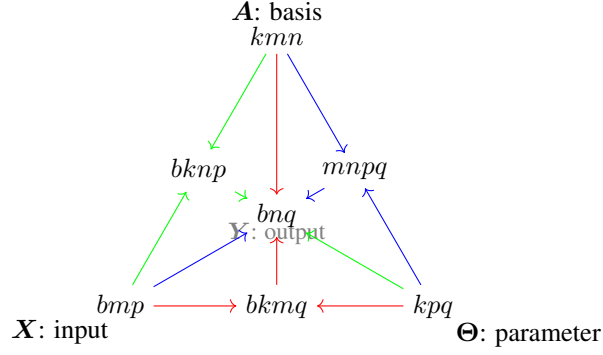
\subsection{Parameter dimension reduction}
Parameter $\boldsymbol{\Theta}$, of shape $\tuple{K,P,Q}$, although controlled, may still be too large and it may be useful to constrain it further. A number of techniques have been proposed to achieve this, in particular:
\begin{itemize}
\item
Grouped convolutions~\cite{krizhevsky_imagenet_2012} constrain each $\boldsymbol{\Theta}_k$ to be a block diagonal matrix of shape $\tuple{P,Q}$. If $\nu$ is the number of blocks, which must be a divisor of both $P$ and $Q$, each block is of shape $\tuple{\frac{P}{\nu},\frac{Q}{\nu}}$, of size $\frac{PQ}{\nu^2}$, so the total parameter size is $\frac{KPQ}{\nu}$ instead of $KPQ$. Furthermore, in the implementation of the convolution operator, the blocks can be processed in parallel.
\item
Depth-wise separable convolutions~\cite{chollet_xception:_2016} constrain $\boldsymbol{\Theta}$ to satisfy
$\boldsymbol{\Theta}_{kpq}{=}\boldsymbol{\Theta}^{\textrm{(1)}}_{kp}\boldsymbol{\Theta}^{\textrm{(2)}}_{pq}$
where $\boldsymbol{\Theta}^{\textrm{(1)}}$ and $\boldsymbol{\Theta}^{\textrm{(2)}}$ are matrices of shape $\tuple{K,P}$ and $\tuple{P,Q}$, respectively. The total parameter size is therefore $KP{+}PQ$ instead of $KPQ$.
\end{itemize}
We propose here another model, called ``controlled separability'', which constrains $\boldsymbol{\Theta}$ to be of the form
\begin{align}
\label{eqn:separable}
\boldsymbol{\Theta} & \eqnrel{=} \boldsymbol{\Theta}^{(\textrm{basis})}\circ\boldsymbol{\Theta}^{(\textrm{channel})}
\end{align}
where $\boldsymbol{\Theta}^{(\textrm{basis})}$ is a matrix of shape $\tuple{H,K}$ (for some integer $H$) and $\boldsymbol{\Theta}^{(\textrm{channel})}$ a tensor of shape $\tuple{H,P,Q}$. It allows to decouple the number $K$ of basis matrices from the number $H$ of parameter matrices of shape $\tuple{P,Q}$. The total parameter size is therefore $H(K{+}PQ)$ instead of $KPQ$, which is useful only if $H{<}K{\ll}PQ$. In normal convolutions, doubling $K$ mechanically doubles the size of $\boldsymbol{\Theta}$. With controlled separability, $H$ can be kept unchanged, so the size of $\boldsymbol{\Theta}^{(\textrm{channel})}$ remains the same, while only doubling the size of the much smaller $\boldsymbol{\Theta}^{(\textrm{basis})}$. Equation~\eqref{eqn:factorisation} becomes
\[
\boldsymbol{\Phi} = {\sum}_{hk}\boldsymbol{\Theta}^{(\textrm{basis})}_{hk}\boldsymbol{A}_k\otimes\boldsymbol{\Theta}^{(\textrm{channel})}_h
\]
which expresses a partial version of the Tucker factorisation scheme~\cite{rabanser_introduction_2017} rather than the CanDecomp scheme captured by Equation~\eqref{eqn:mixed-product}. Controlled separability can also be combined with the other forms of dimension reduction cited above, by simply applying them to $\boldsymbol{\Theta}^{(\textrm{channel})}$ rather than $\boldsymbol{\Theta}$ directly. Yet another form is discussed in Section~\ref{sec:transformer}.
\subsection{A note on the computation of convolutions}
In practice, the input and output entries are usually batched. Batched input $\boldsymbol{X}$ and output $\boldsymbol{Y}$ are given by tensors of shape $\tuple{B,M,P}$ and $\tuple{B,N,Q}$, respectively, where $B$ is the batch size. Figure~\ref{fig:computing} shows the three alternatives to compute $\boldsymbol{Y}$ (at the centre of the triangle) as a function of $\boldsymbol{A},\boldsymbol{X},\boldsymbol{\Theta}$ (on the vertices of the triangle), according to the convolution formula of Equation~\eqref{eqn:convolution} extended to batches:
\[
\boldsymbol{Y}_b = {\sum}_k\boldsymbol{A}_k^\top\boldsymbol{X}_b\boldsymbol{\Theta}_k
\]
Which of these alternatives should be used essentially depends on the respective dimensions $M,P,N,Q,B,K$. In any case, operations involving the basis tensor $\boldsymbol{A}$ may require a specific treatment, since it is usually very sparse, and sometimes possesses a regularity which can be exploited for optimal computation, as in the case of the ``shift matrices'' of grid convolution (Figure~\ref{fig:grid-shift}).
\section{Some examples}
\subsection{Grid convolutions}
\label{sec:grid}
\begin{figure}
\begin{center}
\begin{tabular}{c@{\hspace{1.5cm}}c}
\begin{tabular}{c}
\begin{tikzpicture}[scale=.5]
\foreach \i in {1,...,10} { \node[label={[shift={(-.25,0)}]0:\tiny\i}] at (\i,0) {\tiny $\bullet$}; }
\foreach \i in {1,2,5,6} { \draw [->,blue] let \n{n}={\i+2} in (\i,0) to[bend right=90] (\n{n},0); }
\foreach \i in {3,4,7,8} { \draw [->,blue] let \n{n}={\i+2} in (\i,0) to[bend left=90] (\n{n},0); }

\end{tikzpicture}
\\
\includegraphics[scale=.22]{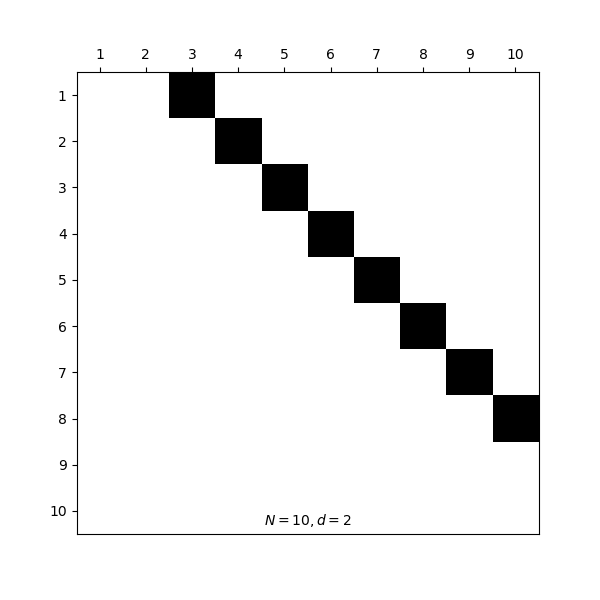}
\end{tabular}
&
\begin{tabular}{c@{\hspace{.5cm}}c}
\begin{tikzpicture}[scale=.5]
\foreach \i in {1,...,8} { \foreach \j in {1,...,10} { \node at (\j,\i) {\tiny $\bullet$}; } }
\foreach \i in {3,...,8} { \foreach \j in {1,...,6} { \draw [->,blue] let \n{i}={\i-2},\n{j}={\j+4} in (\j,\i) to[bend right=65] (\n{j},\n{i}); } }
\node at (1.3,1.2) {\tiny $71$}; \node at (9.7,1.2) {\tiny $80$};
\node at (1.3,7.8) {\tiny $1$}; \node at (9.7,7.8) {\tiny $10$};

\end{tikzpicture}
&
\includegraphics[scale=.3]{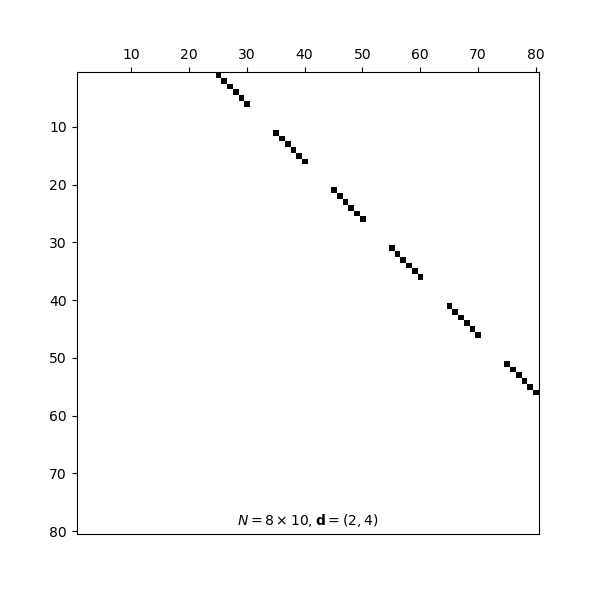}
\end{tabular}
\end{tabular}
\end{center}
\caption{\label{fig:grid-shift}Shift matrices constitute the basis of grid convolutions. Left: shift by $2$ in a 1-D grid of dimension $10$; Right: shift by $(2,4)$ in a 2-D grid of dimensions $8{\times}10$ (flattened by the canonical mapping into the sequence $1{:}80$).}
\end{figure}
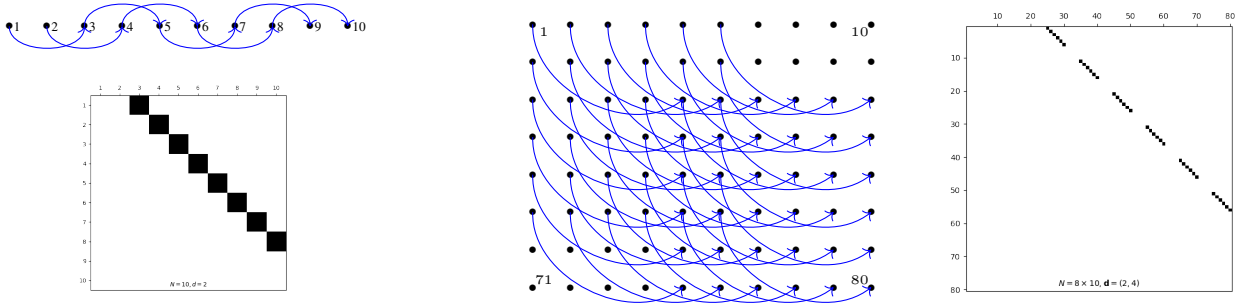
A grid is the index set $\bar{S}$ associated with a given sequence of integers $S$. In the case of images, the archetypal grids, $S$ is the sequence $\tuple{\textrm{width},\textrm{height}}$ of length $|S|{=}2$. We assume given some bijective mapping $\omega{:}\bar{S}{\mapsto}\{1\cdots N\}$ where $N{=}|\bar{S}|$, for example the canonical mapping $\omega_S$ (see Section~\ref{sec:tensors}). In this way, an embedding of the whole grid, which would naturally be represented by a tensor of shape $S\tuple{L}$ where each node in the grid is encoded as a vector of shape $\tuple{L}$, can be matricised (see Section~\ref{sec:tensors}) into a matrix of shape $\tuple{N,L}$ as used in our model. Let's first consider convolutions which preserve the grid, hence $M{=}N$.
\begin{definition}
For each integer valued vector $\boldsymbol{d}{\in}\mathbb{Z}^{|S|}$, we define the {\em shift matrix} $\mathcal{A}_{\boldsymbol{d}}$ of shape $\tuple{N,N}$ by
\[
(\mathcal{A}_{\boldsymbol{d}})_{mn} \triangleq \mathbb{I}[\omega^{-1}n-\omega^{-1}m=\boldsymbol{d}]
\]
A {\em grid convolution} of size $K$ and basis $\boldsymbol{A}$ is one such that for each $k{\in}1{:}K$, $\boldsymbol{A}_k{=}\mathcal{A}_{\boldsymbol{\Delta}_k}$ for some $\boldsymbol{\Delta}_k{\in}\mathbb{Z}^{|S|}$.
\end{definition}
Thus, $\mathcal{A}_{\boldsymbol{d}}$ is the adjacency matrix of the relation: ``node $n$ is obtained from node $m$ by a shift of $\boldsymbol{d}$ in the grid''. It is illustrated in Figure~\ref{fig:grid-shift} in the case of grids of order $|S|{=}1,2$ (typically, sentences and images). With some padding conventions, Equation~\eqref{eqn:convolution} for a grid convolution becomes, for any node $s{\in}\bar{S}$ in the grid,
\[
\boldsymbol{y}_{(\omega(s))} =
{{\sum}}_k\boldsymbol{x}_{(\omega(s-\boldsymbol{\Delta}_k))}\boldsymbol{\Theta}_k
\]
The traditional grid (image) convolutions of ``Convolutional Neural Networks'' (CNNs)~\cite{lecun_convolutional_1998} are exactly obtained by choosing $\boldsymbol{\Delta}$ to be a regular right cuboid with possibly different strides $\boldsymbol{\delta}_i$ and offsets $\boldsymbol{\epsilon}_i$ in the different grid dimensions $i{=}1{:}|S|$. In that case, we have $K=\prod_{i=1:|S|}K_i$ and for each $k{\in}1{:}K$
\begin{align}
\label{eqn:cnn-shift}
\boldsymbol{\Delta}_{ki} & \eqnrel{=} \boldsymbol{\epsilon}_i+(\omega_{\tuple{K_1\cdots K_{|S|}}}^{-1}k)_i\boldsymbol{\delta}_i
\end{align}
Parameter $\boldsymbol{\Theta}$, of shape $\tuple{K,P,Q}$, then appears as the flattened version of a tensor of shape $\tuple{K_1,\cdots,K_{|S|},P,Q}$, which is the familiar shape of grid convolution kernels.

Variants of grid convolutions which do not necessarily preserve the grid can also be captured in our framework using different choices of $\boldsymbol{\Delta}$, and variants of the shift matrices. This includes average pooling and dilated convolutions, where the output grid is a sub-sample of the input one ($N$ is a divisor of $M$ rather than $M{=}N$). However, our framework covers only convolutions which are linear transforms, which rules out such things as max pooling.

The choice of basis matrices introduced above captures exactly the conditions to ensure that the resulting convolutions satisfy two a priori constraints: {\em translation equivariance} and {\em locality}. To show that, recall that, by Proposition~\ref{prop:inversion}, any linear transform $\boldsymbol{\Phi}$ can be written in the form $\boldsymbol{a}\circ\boldsymbol{\Theta}$ where $(\boldsymbol{a}_h)_{h=1:MN}$ is a basis of the space of matrices of shape $\tuple{M,N}$. One obvious such basis is given by
\[(\boldsymbol{a}_h)_{mn}\triangleq\mathbb{I}[\tau(m,n)=h]\]
where $\tau$ is a given bijection $\{1{\cdots}M\}{\times}\{1{\cdots}N\}{\mapsto}\{1{\cdots}MN\}$, e.g. the canonical bijection $\omega_{\tuple{M,N}}$. Thus, when $h{=}\tau(m,n)$, the term $\boldsymbol{a}_h^\top\boldsymbol{x}\boldsymbol{\Theta}_h$ in Equation~\eqref{eqn:convolution} can be understood as follows: $\boldsymbol{a}_h$ filters the action of node $m$ in the input grid onto node $n$ in the output grid, and $\boldsymbol{\Theta}_h$ specifies the linear transform which must be applied to the input embedding of $m$ to obtain its contribution to the output embedding of $n$.
\begin{itemize}
\item
Translation equivariance means that the action of $m$ on $n$ should be the same as the action of $m'$ on $n'$ where $m',n'$ are obtained from $m,n$ by the same translation, i.e., borrowing a term from geometry, the two pairs $h{=}\tau(m,n)$ and $h'{=}\tau(m',n')$ are {\em equipollent}. In other words, translation equivariance amounts to pooling together the parameter $\boldsymbol{\Theta}_h$ of all the pairs $h{=}\tau(m,n)$ which are equipollent. This amounts to regrouping (summing together) all the corresponding basis matrices $\boldsymbol{a}_{h}$, yielding exactly what is called above a shift matrix.
\item 
By itself, translation equivariance constrains the basis matrices to be shift matrices but does not constrain the size of the shifts, so that a node anywhere in the input grid could still act on any node of the output grid. Locality is achieved by further constraining the shift vectors outside a small neighbourhood of the null vector, as defined by Equation~\eqref{eqn:cnn-shift}, to have a null contribution.
\end{itemize}
\subsection{Graph convolutions}
Let $\mathcal{G}$ be a graph over $\{1\cdots N\}$ given a priori. We assume $M{=}N$ (graph convolutions usually preserve the graph).
\begin{definition}
A {\em graph convolution} of size $K$ and basis $\boldsymbol{A}$ is one such that for each $k{\in}1{:}K$, matrix $\boldsymbol{A}_k$ is constructed from $\mathcal{G}$ by some procedure dependent on $k$.
\end{definition}
The traditional ``Graph Convolution Networks'' (GCNs)~\cite{kipf_semi-supervised_2016} are exactly obtained by choosing $K{=}1$ and $\boldsymbol{A}_1$ to be the normalised Laplacian matrix of $\mathcal{G}$. Constraining the size to $1$ yields a very simple, efficient architecture, at the price of some expressiveness. For example, although grids can be represented as graphs, grid convolutions cannot be expressed as graph convolutions with a size restricted to $1$.

In alternative definitions of graph convolution, the size is possibly greater than $1$, and each $\boldsymbol{A}_k$ is computed from $\mathcal{G}$ in a different way. For example, in the full spectral analysis of graph convolution~\cite{defferrard_convolutional_2016}, each $\boldsymbol{A}_k$ is a Chebyshev polynomial of the normalised Laplacian matrix of $\mathcal{G}$, up to order $K$. In a simpler version~\cite{li_diffusion_2017}, Chebyshev polynomials are replaced by elementary monomials, and $\boldsymbol{A}_k$ is simply the adjacency matrix of $\mathcal{G}$ raised to the power of $k$, capturing the random walks of length $k$ through the graph.

A similar approach can be applied to knowledge graphs~\cite{schlichtkrull_modeling_2017}, by introducing one basis matrix $\boldsymbol{A}_k$ for each random walk sort (instead of length) from a given set of sorts (instead of up to a given length), where a sort is a sequence of relations. For example, in a film knowledge base, a sort could be ``{\tt played.characterIn.genre}'', and a typical instance of random walk of that sort could be ``{\tt LeonardNimoy-Spock-StarTreck-SciFi}''.
\section{Attention as content-based convolution}
\subsection{Content-based vs index-based convolution}
In the previous examples of convolution, the basis tensor captures prior knowledge about the structural relationships between input and output entries through their indices. This is not the only option. Instead of relying solely on indices, the basis tensor of a convolution can also be computed from any content associated with the input and output entries. We propose a generic model to achieve this, and claim that it captures the essence of many attention mechanisms:~\cite{vaswani_attention_2017,elbayad_pervasive_2018,velickovic_graph_2017,cinar_period-aware_2018,cinar_position-based_2017,kool_attention_2018}.
\begin{definition}
An {\em attention mechanism} is a parametrised mapping which takes as input two matrices, of shape $\tuple{M,P'}$ and $\tuple{N,Q'}$, respectively, and returns an output matrix of shape $\tuple{M,N}$. The input matrices represent $M$ and $N$ entries encoded as vectors of shape $\tuple{P'}$ and $\tuple{Q'}$, respectively, and the output matrix represents an influence graph of the former on the latter, based on their encodings.
\end{definition}
Attention mechanisms can be added or multiplied term-wise, or transformed by term-wise, row-wise or column-wise normalisation. Two particularly useful transformations are {\em masking} and (column-wise or row-wise) {\em softmax} normalisation, often used in conjunction. Masking is described here in log domain: given a mask as an a priori matrix $\boldsymbol{H}$ of shape $\tuple{M,N}$ with values in $\{-\infty,0\}$ (the $\log$ of a binary matrix), if $a$ is an attention mechanism, then one can straightforwardly form the mechanism $a{+}\boldsymbol{H}$: it masks (sets to $-\infty$) the output of $a$ wherever $\boldsymbol{H}$ is $-\infty$ leaving the other values unchanged. In particular, if $\boldsymbol{H}$ is sparse, i.e. the density of $-\infty$ is high, then $a{+}\boldsymbol{H}$ is also sparse, i.e. the density of masked values is high, whatever the sparseness of $a$. This is useful when the dimensions $M,N$ are large and the size $MN$ of the output of $a$ becomes unmanageable. Masking allows to limit a priori which entries from the first input can influence entries from the second input. Observe that when softmax normalisation is applied to a masked attention, the masked values become $0$, cancelling the influence of the corresponding inputs in the linear domain\footnote{As a general rule, softmax takes input in log domain (scores) and produces output in linear domain (probabilities).}.
\begin{definition}
An {\em attention convolution} of size $K$ and basis $\boldsymbol{A}$ is one such that for each $k{=}1{:}K$, $\boldsymbol{A}_k{=}a(\boldsymbol{x}',\boldsymbol{y}';\boldsymbol{\Xi}_k)$ for some attention mechanism $a$ and some $\boldsymbol{\Xi}_k$ in the parameter space of $a$. The convolution now has three input matrices, the main input $\boldsymbol{x}$ of shape $\tuple{M,P}$, and two auxiliary inputs $\boldsymbol{x}',\boldsymbol{y}'$ of shape $\tuple{M,P'},\tuple{N,Q'}$ respectively, and returns an output matrix $\boldsymbol{y}$ of shape $\tuple{N,Q}$ according to Equation~\eqref{eqn:convolution}, which can be rewritten:
\begin{align}
\label{eqn:attention}
\boldsymbol{y} & =
{\sum}_ka(\boldsymbol{x}',\boldsymbol{y}';\boldsymbol{\Xi}_k)^\top\boldsymbol{x}\boldsymbol{\Theta}_k
\end{align}
In {\em cross-attention} (resp. {\em self-attention}) convolutions, the main input $\boldsymbol{x}$ is also used as the auxiliary input $\boldsymbol{x}'$ (resp. as both $\boldsymbol{x}',\boldsymbol{y}'$), which, to be shape-consistent, requires $P'{=}P$ (resp. $N{=}M$ and $P'{=}Q'{=}P$).
\begin{center}
\begin{tikzpicture}
\tikzset{
port/.style={circle,minimum height=1.3mm,inner sep=0,fill=black},
var/.style={fill=white},
bbox/.style={fill=black!5}
}
\draw[bbox] (0.,.3) rectangle (1.7,-1.3);
\node[port] (inM) at (0.,0.) {};
\node[port] (in1) at (0.,-.5) {};
\node[port] (in2) at (0.,-1.) {};
\node[port] (out) at (1.7,0.) {};
\node at (.85,-1.5) {attention};
\draw (-.5,0.) node[var] {$\boldsymbol{x}$} -- (0.,0.);
\draw (-.5,-.5) node[var] {$\boldsymbol{x}'$} -- (0.,-.5);
\draw (-.5,-1.) node[var] {$\boldsymbol{y}'$} -- (0.,-1.);
\draw (out) -- (2.2,0.) node[var] {$\boldsymbol{y}$};
\node[draw,rectangle,minimum height=.8cm] (am) at (.5,-.75) {$a$};
\node (A) at (1.2,-.75) {$\boldsymbol{A}$};
\node[draw,rectangle] (cv) at (1.2,0.) {$*$};
\draw[->] (in1.east) -- (in1-|am.west); \draw[->] (in2.east) -- (in2-|am.west); \draw[->] (am) -- (A);
\draw[->] (A) -- (cv); \draw[->] (inM) -- (cv); \draw[->] (cv) -- (out);
\begin{scope}[xshift=4cm]
\draw[bbox] (0.,.3) rectangle (.5,-1.3);
\node[port] at (0.,0.) {};
\node[port] at (0.,-.5) {};
\node[port] at (0.,-1.) {};
\node[port] at (.5,0.) {};
\node at (.25,-1.5) {cross-attention};
\draw (-.5,0.) node[var] {$\boldsymbol{x}$} -- (0.,0.);
\draw (-.15,0.) rectangle (0.,-.5);
\draw (-.5,-1.) node[var] {$\boldsymbol{y}'$} -- (0.,-1.);
\draw (.5,0.) -- (1.,0.) node[var] {$\boldsymbol{y}$};
\end{scope}
\begin{scope}[xshift=6.5cm]
\draw[bbox] (0.,.3) rectangle (.5,-1.3);
\node[port] at (0.,0.) {};
\node[port] at (0.,-.5) {};
\node[port] at (0.,-1.) {};
\node[port] at (.5,0.) {};
\node at (.25,-1.5) {self-attention};
\draw (-.5,0.) node[var] {$\boldsymbol{x}$} -- (0.,0.);
\draw (-.15,0.) rectangle (0.,-.5);
\draw (-.15,0.) rectangle (0.,-1.);
\draw (.5,0.) -- (1.,0.) node[var] {$\boldsymbol{y}$};
\end{scope}

\end{tikzpicture}
\end{center}
\end{definition}
A graph or grid convolution, as described in the previous sections, can be seen as a degenerate case of attention convolution, in which the output of the mechanism does not depend on its input, but solely on its parameter, given a priori (not learnt). The resulting convolution thus ignores its auxiliary inputs and is linear in its main input. On the other hand, in the non degenerate case, an attention convolution may be non linear in either of its auxiliary inputs, depending on the mechanism. Furthermore, a self- or cross-attention convolution may not even be linear in its main input, since it also occurs as input to the mechanism. A commonly used attention mechanism is bi-affine attention:
\begin{definition}
Let $\xi$ be a scalar, $\boldsymbol{\mu},\boldsymbol{\nu}$ be vectors of shape, respectively, $\tuple{P'},\tuple{Q'}$, and $\boldsymbol{\Lambda}$ be a matrix of shape $\tuple{P',Q'}$. The {\em bi-affine} attention mechanism $\mathcal{A}$ parametrised by $\boldsymbol{\Xi}{=}\tuple{\xi,\boldsymbol{\mu},\boldsymbol{\nu},\boldsymbol{\Lambda}}$ is defined, for matrices $\boldsymbol{x}',\boldsymbol{y}'$ of shape, respectively, $\tuple{M,P'}$ and $\tuple{N,Q'}$, by
\begin{align}
\label{eqn:biaffine-attention}
\mathcal{A}(\boldsymbol{x}',\boldsymbol{y}';\boldsymbol{\Xi}) & \eqnrel{\triangleq}
\boldsymbol{x}'\boldsymbol{\Lambda}\boldsymbol{y}'^T+
(\boldsymbol{x}'\boldsymbol{\mu})\otimes\mathbf{1}_{N}+
\mathbf{1}_{M}\otimes(\boldsymbol{y}'\boldsymbol{\nu})+
\xi\mathbf{1}_{M}\otimes\mathbf{1}_{N}\\
\nonumber
\textrm{equivalently,}\hspace{1cm}
\mathcal{A}(\boldsymbol{x}',\boldsymbol{y}';\boldsymbol{\Xi})_{mn}  & \eqnrel{=}
{\sum}_{p'q'}\boldsymbol{\Lambda}_{p'q'}\boldsymbol{x}'_{mp'}\boldsymbol{y}'_{nq'}+
{\sum}_{p'}\boldsymbol{\mu}_{p'}\boldsymbol{x}'_{mp'}+
{\sum}_{q'}\boldsymbol{\nu}_{q'}\boldsymbol{y}'_{nq'}+
\xi
\end{align}
\end{definition}
Observe that parameter $\boldsymbol{\Xi}$ has a fully controlled shape, independent of $M,N$. Bi-affine attention is used in the specific context of parsing in~\cite{dozat_deep_2017}. It is also used, with some restrictions on parameter $\boldsymbol{\Xi}$, as a generic attention mechanism in the Transformer model for sequences~\cite{vaswani_attention_2017}, and in Graph attention networks~\cite{velickovic_graph_2017}, as shown below.
\subsection{Attention in Graph Attention Networks}
Graph attention networks~\cite{velickovic_graph_2017} are based on a variant of self-attention convolutions, where the equation $\boldsymbol{A}_k{=}a(\boldsymbol{x},\boldsymbol{x};\boldsymbol{\Xi}_k)$ is replaced by $\boldsymbol{A}_k{=}a(\boldsymbol{x}\boldsymbol{\Theta}_k,\boldsymbol{x}\boldsymbol{\Theta}_k;\boldsymbol{\Xi}_k)$. This does not significantly alter forward computation, at least when started from the bottom of the triangle in Figure~\ref{fig:computing}, since the term $\boldsymbol{x}\boldsymbol{\Theta}_k$ is already available before entering the attention mechanism.

The attention mechanism proposed in~\cite{velickovic_graph_2017} starts with the bi-affine mechanism of Equation~\eqref{eqn:biaffine-attention} without its bi-linear term, i.e. $\boldsymbol{\Lambda}_k{=}0$ for all head $k$. The output is then masked by a graph given a priori, the same for all heads, limiting the set of nodes attending on a given node to a neighbourhood of that node. When such prior graph is available, this makes sense, esp. to deal with large structures such as publication networks (up to 50,000 nodes in their experiments, hence, without mask, the output of the mechanism would be of size of magnitude $10^9$).

The masked output is then normalised by a term-wise ``leaky ReLU'' followed by a column-wise softmax. These choices can be motivated to some extent by properties of the simplified bi-affine mechanism at work. Indeed, observe that the masked values are still masked after a leaky ReLU\footnote{``Leaky'' is important here: recall that, in log domain, the mask value is $-\infty$ which is unchanged by a leaky ReLU but annulled by a plain ReLU.}, and are annulled by the softmax, cancelling as intended the influence of the corresponding inputs. Skipping ReLU altogether before the softmax would make the term involving $\boldsymbol{\nu}$ in the right-hand side of Equation~\eqref{eqn:biaffine-attention} redundant: it is constant along each column, and softmax is invariant to an additive constant.
\subsection{Attention in Transformer}
\label{sec:transformer}
We now show how the attention model described by Equations~\eqref{eqn:attention} and~\eqref{eqn:biaffine-attention} encompasses the scaled dot product attention used in the Transformer model of~\cite{vaswani_attention_2017}. In Transformer attention convolutions, the auxiliary inputs are called ``key'' and ``query'', respectively, while the main input is called ``value''. Attention is used in three distinct layers of the Transformer architecture. Two of them are instances of self-attention (on the source sequence and on the target sequence, respectively) while the third one is a cross-attention (the main input is the source sequence and the remaining auxiliary input is the target sequence). Masking is used in the target sequence self-attention, to ensure that tokens in that sequence do not have influence on their predecessors. This is because, in Transformer, the ultimate output of all the attention layers is used to model the next token from each position in the target sequence, so should not rely on the availability of that token.

In all cases, the scaled dot product attention used in Transformer essentially relies on the bi-affine attention mechanism of Equation~\eqref{eqn:biaffine-attention}, followed by a column-wise softmax. Actually, only the bi-linear part of Equation~\eqref{eqn:biaffine-attention} is kept, i.e. all the parameters except $\boldsymbol{\Lambda}_k$ are null. Furthermore, parameter $\boldsymbol{\Lambda}_k$ is constrained to be of the form
\begin{align}
\label{eqn:transformer-lambda}
\boldsymbol{\Lambda}_k & \eqnrel{=} \boldsymbol{\Lambda}_k^{(\textrm{key})}\boldsymbol{\Lambda}_k^{(\textrm{query})\top}
\hspace{.5cm}\left(=\boldsymbol{\Lambda}_k^{(\textrm{key})\top}\circ\boldsymbol{\Lambda}_k^{(\textrm{query})\top}\right)
\end{align}
where matrices $\boldsymbol{\Lambda}_k^{(\textrm{key})},\boldsymbol{\Lambda}_k^{(\textrm{query})}$ are of shape $\tuple{P',D},\tuple{Q',D}$, respectively. This can be viewed as a simple dimension reduction technique, since only $(P'{+}Q')D$ parameters are required instead of $P'Q'$ for an arbitrary $\boldsymbol{\Lambda}_k$.

Now, Transformer attention introduces a seemingly richer mechanism to combine the different heads. Instead of simply summing them together as in Equation~\eqref{eqn:attention}, it combines them with yet another linear layer:
\[
\boldsymbol{y} = [\boldsymbol{h}_1,\ldots,\boldsymbol{h}_K]\boldsymbol{\Theta}^{(\textrm{O})\top}
\hspace{1cm}\textrm{where}\hspace{1cm}
\boldsymbol{h}_k\triangleq\boldsymbol{A}_k^\top\boldsymbol{x}\boldsymbol{\Theta}_k^{(\textrm{value})}
\]
where $\boldsymbol{\Theta}^{(\textrm{O})}$ is a matrix of shape $\tuple{Q,KD}$ and each $\boldsymbol{\Theta}_k^{(\textrm{value})}$ is a matrix of shape $\tuple{P,D}$. In fact, this expression can be rewritten, splitting $\boldsymbol{\Theta}^{(\textrm{O})}$ into $K$ blocks $(\boldsymbol{\Theta}^{(\textrm{O})}_k)_{k=1:K}$ each of shape $\tuple{Q,D}$, as
\[
\boldsymbol{y} = {\sum}_k\boldsymbol{h}_k\boldsymbol{\Theta}^{(\textrm{O})\top}_k = {\sum}_k\boldsymbol{A}_k^\top\boldsymbol{x}\boldsymbol{\Theta}_k^{(\textrm{value})}\boldsymbol{\Theta}^{(\textrm{O})\top}_k
\]
In other words, it is strictly equivalent to the sum model of Equation~\eqref{eqn:attention}, only with the constraint
\begin{align}
\label{eqn:transformer-theta}
\boldsymbol{\Theta}_k & \eqnrel{=}\boldsymbol{\Theta}_k^{(\textrm{value})}\boldsymbol{\Theta}^{(\textrm{O})\top}_k
\hspace{.5cm}\left(=\boldsymbol{\Theta}_k^{(\textrm{value})\top}\circ\boldsymbol{\Theta}^{(\textrm{O})\top}_k\right)
\end{align}
This constraint is not even specific to an attention model and could apply to any convolution. In fact, Equations~\eqref{eqn:transformer-theta} and~\eqref{eqn:transformer-lambda} are meant to reduce the dimensionality of the parameters ($\boldsymbol{\Theta}$, and, in the case of attention,  $\boldsymbol{\Lambda}$) by factorisation. Formally, they apply the exact same recipe as applied to $\boldsymbol{\Phi}$ in Equation~\eqref{eqn:factorisation} or $\boldsymbol{\Theta}$ in Equation~\eqref{eqn:separable}, with the same purpose.

Finally, Transformers take the extreme approach of relying exclusively on content-based convolution (``attention is all you need''), so that any index-based information such as the relative position of the tokens must be incorporated into the content. They propose a smart but not completely intuitive scheme to achieve that, called ``positional encoding''. Instead, one or several additional heads with purely index-based basis matrices could be used, in complement to the attention heads. Typically, the index-based basis matrices would be those of a 1-D grid convolution, which have a particularly simple form (see Figure~\ref{fig:grid-shift}). And actually, it has been observed that in Transformers with positional encoding, after training, some of the attention heads precisely play the role of shift matrices. In the end, the problem is to model the role of token ordering in a sentence. It looks more natural to model it by directly adding shift matrices in the basis, than through a complicated encoding, validated mainly by the fact that some of the resulting attention heads end up playing the role of shift matrices in the basis...
\newpage
\printbibliography
\end{document}